%% file: collas2026_conference.tex
\documentclass{article} 
\usepackage{collas2026_conference,times}
\usepackage{easyReview}

\input{math_commands.tex}

\usepackage{subcaption}
\usepackage{booktabs} 
\usepackage{bm}

\usepackage{hyperref}
\hypersetup{
    colorlinks=true,
    linkcolor=red,
    filecolor=magenta,
    urlcolor=blue,
    citecolor=purple,
    pdftitle={Overleaf Example},
    pdfpagemode=FullScreen,
    }

\title{Explaining $f$-Divergence-Based Regularization via Local Curvature and Sharpness-Aware Minimization}

\author{Nour Jamoussi \\
Communication Systems Department\\
EURECOM \\
France \\
\texttt{nour.jamoussi@eurecom.fr} \\
\And 
Marios Kountouris  \\
Department of Computer Science and Artificial Intelligence, \\
University of Granada, Spain \\
Communication Systems Department, EURECOM, France \\
\texttt{mariosk@ugr.es} \\
}

\collasfinalcopy 

\begin{document}

\maketitle

\begin{abstract}
Divergence-based regularization and Sharpness-Aware Minimization (SAM) are two prominent approaches for improving generalization in deep learning, both motivated by robustness to perturbations. However, their relationship has remained largely unexplored. Building on classical second-order expansions of $f$-divergences, we show that the two methods are locally consistent under parameter-space perturbations: both induce curvature-sensitive penalties, with divergence regularization yielding a Fisher-weighted quadratic form and SAM penalizing sharpness through the dominant Hessian eigenvalue. For negative log-likelihood objectives with exponential-family output distributions, this correspondence becomes especially transparent, since the Fisher and Gauss-Newton matrices coincide. We further show that the same local geometric perspective extends to input-space perturbations, where divergence-based regularization is defined through transformations of the input. In this setting, the regularizer induces a pullback quadratic form on the input space, providing a more general perturbation framework than standard SAM while preserving the same local sensitivity interpretation. To validate the analysis empirically, we use the asymmetric $\alpha$-skew Jensen-Shannon divergence (JSD) family as a controlled testbed. Its local curvature coefficient scales as $\alpha(1-\alpha)$ and is maximized at the symmetric point $\alpha=\tfrac12$, which recovers the standard JSD. Loss-landscape visualizations in the input-perturbation regime show that stronger induced curvature penalization is associated with flatter local minima. Experiments on four benchmark datasets further demonstrate that both accuracy and negative log-likelihood are consistently best near this regime of maximal curvature penalization.
\end{abstract}

\section{Introduction}

For lifelong learning agents, generalization is not limited to performance on a fixed held-out distribution; it also requires sustaining reliable predictions as the agent encounters a continuing stream of changing observations, tasks, and environments. This makes controlling overfitting and local sensitivity especially important: a model that fits the current experience too sharply may fail to transfer to future observations or become brittle under the small distributional shifts that arise over time.

More broadly, generalization is a central objective in machine learning: a model should not merely fit the training data, but also perform reliably on unseen examples. Achieving this requires understanding and controlling the factors that contribute to overfitting. A rich body of work has been developed to address this challenge, spanning data augmentation~\citep{NIPS2012_c399862d,zhang2017mixup,shorten2019survey} and regularization~\citep{hoerl1970ridge,tibshirani1996regression,srivastava2014dropout} methods designed to favor robust solutions.

Divergence-based regularization~\citep{miyato2015distributional,xie2020unsupervised,wei2020optimizing,li2025regularization,sevillano2026x} occupies a distinctive position at the intersection of these two paradigms: by penalizing discrepancies between a model's predictive distributions on the original input and its perturbed counterpart, it simultaneously encourages output consistency under input transformations and imposes a form of distributional smoothness. In this sense, it can be understood as a combination of implicit data augmentation and explicit regularization, operating directly on the output distribution.

A complementary perspective on generalization arises from the geometry of the loss landscape. Sharpness-Aware Minimization (SAM)~\citep{foret2020sharpness} is grounded in the observation that flat minima tend to generalize better than sharp ones. SAM formalizes this intuition by solving a min-max problem that seeks parameters whose entire neighborhood exhibits low training loss, thereby explicitly penalizing curvature in parameter space.

Despite their shared motivation, the relationship between divergence-based regularization and SAM has not been fully characterized from a local second-order geometric perspective. In particular, it remains unclear whether the local geometry induced by a divergence regularizer can be related to the Hessian structure that SAM implicitly controls.

This paper addresses this gap by analyzing divergence-based regularization through the lens of local second-order geometry and by deriving a formal connection to SAM. Our main contributions are as follows:

\begin{itemize}
    \item We derive a unified local second-order expansion for $f$-divergence regularization, showing that small perturbations induce quadratic penalties determined by Fisher geometry in parameter space and by a pullback metric in input space. 
    \item We show that, in the parameter-perturbation setting and under standard probabilistic-loss assumptions, this geometry becomes locally comparable to second-order interpretations of SAM through the Fisher/Generalized Gauss-Newton/Hessian relationship.
    \item Using the asymmetric $\alpha$-skew Jensen-Shannon divergence (JSD) family as a controlled example, we derive the local curvature coefficient and empirically show how the symmetric regime relates to predictive performance and local landscape flatness. Our empirical study focuses primarily on the input-perturbation regime; accordingly, it validates the broader local-curvature perspective rather than a direct experimental equivalence to SAM.
\end{itemize}

The remainder of the paper is organized as follows. Section~\ref{section:related_work} reviews the related literature. In Section~\ref{section:general_framework} we present the general framework for $f$-divergence regularization and its local quadratic geometry, including an analysis of a controlled example based on the asymmetric $\alpha$-skew JSD family. In Section~\ref{section:local_connection_to_sam} we derive the local connection between $f$-divergence regularization and SAM via an analysis of the curvature component of the loss. In Section~\ref{section:experiments} we present empirical validation of the effect of divergence-based regularization curvature on the loss landscape and model performance. Section~\ref{section:limitations} discusses the limitations of this analysis for divergence selection in regularization.

\section{Related Work}
\label{section:related_work}

\paragraph{Divergence-based regularization and information-geometric penalties.}
A first line of related work focuses on regularizing learning by penalizing discrepancies between predictive distributions under perturbations, transformations, or auxiliary constraints. In particular, KL-based and symmetric-KL-based penalties have been widely used to enforce output consistency under perturbed inputs or adversarial examples, as in virtual adversarial training~\citep{miyato2015distributional}, unsupervised data augmentation~\citep{xie2020unsupervised} and explainability improvements ~\citep{sevillano2026x}.  Beyond KL-based constructions, more general $f$-divergence regularizers have also been explored, for example in fairness-constrained learning~\citep{zhong2023learning} and learning with noisy supervision~\citep{wei2020optimizing,li2025regularization}. A related information-geometric perspective is provided by Fisher-Rao regularization~\citep{picot2022adversarial}, which penalizes distributional changes through the Fisher-Rao metric in order to improve adversarial robustness.

\paragraph{Sharpness-aware minimization.}
A second line of related work studies robustness to parameter perturbations through sharpness-aware optimization. SAM~\citep{foret2020sharpness} formulates training as the minimization of the worst-case loss in a neighborhood of the parameters, leading to a min--max objective that favors flat regions of the loss landscape. Subsequent work has refined this perspective in several directions. Fisher SAM~\citep{kim2022fisher} replaces the Euclidean neighborhood of SAM with a geometry adapted to the Fisher metric, thereby incorporating information-geometric structure into the sharpness-aware perturbation set. CR-SAM~\citep{wu2024cr} augments SAM with an explicit curvature regularizer based on Hessian trace information.
On the theoretical side, recent analyses~\citep{wen2022does,agarwala2023sam} have clarified the notion of sharpness implicitly induced by SAM and its relation to Hessian eigenvalues during training.

Our work lies at the intersection of these two lines of research. By studying the local second-order geometry induced by divergence-based regularization, we clarify its connection to sharpness-aware minimization from a local geometric perspective. 
Unlike Fisher SAM, which changes the perturbation geometry inside a SAM objective, our analysis starts from divergence regularization and studies the geometry it induces locally. Unlike Fisher-Rao regularization, which directly imposes an information-geometric penalty, we derive a general local expansion valid for $f$-divergences and relate its parameter-space form to second-order interpretations of SAM. Our contribution is therefore primarily analytical: we clarify when divergence-based regularization, Fisher geometry, and sharpness-aware objectives become locally comparable.

\section{A General Framework for $f$-Divergence Regularization}
\label{section:general_framework}
Let $\mathcal{S} = \{(x_i, y_i)\}_{i=1}^n$ be a training dataset drawn i.i.d.\ from a distribution $\mathcal{D}$. The objective is to learn a model that generalizes well beyond this sample. To this end, we consider a family of models $f_\theta : \mathcal{X} \to \mathcal{P}(\mathcal{Y})$ parameterized by $\theta \in \Theta \subseteq \mathbb{R}^d$, together with a per-sample loss function $\ell : \Theta \times \mathcal{X} \times \mathcal{Y} \to \mathbb{R}_+$.

Based on this setup, we define the empirical training loss as
$
L_{\mathcal{S}}(\theta) \triangleq \frac{1}{n} \sum_{i=1}^n \ell(\theta, x_i, y_i),
$
and the population loss as
$
L_{\mathcal{D}}(\theta) \triangleq \mathbb{E}_{(x,y)\sim \mathcal{D}}[\ell(\theta,x, y)].
$ While the population loss characterizes generalization performance, it cannot be computed directly since $\mathcal{D}$ is unknown and only the sample $\mathcal{S}$ is observed. Consequently, learning proceeds by minimizing the empirical loss $L_{\mathcal{S}}(\theta)$ as a proxy for the population loss $L_{\mathcal{D}}(\theta)$.

However, directly minimizing $L_{\mathcal{S}}(\theta)$ may lead to overfitting. To mitigate this issue and promote generalization, it is standard to augment the optimization objective with a regularization term. In its most general form, the resulting learning problem can be written as
\begin{align}
    \min_{\theta \in \Theta} \; L_{\mathcal{S}}(\theta) + \lambda \, \mathcal{R}_{\mathcal{S}}(\theta),
\end{align}
where $\mathcal{R}_{\mathcal{S}} : \Theta \to \mathbb{R}_+$ denotes a regularization functional and $\lambda \geq 0$ is a hyperparameter controlling the trade-off between data fitting and the strength of regularization.

The regularization term may be independent of the training data, in which case it is typically written as $\mathcal{R}(\theta)$, omitting any explicit dependence on $\mathcal{S}$; common examples include ridge and Lasso regularization~\citep{hoerl1970ridge, tibshirani1996regression}. In other settings, however, the regularization term explicitly depends on the data, for instance by acting on the model outputs evaluated on the inputs, as in approaches based on the Fisher-Rao metric~\citep{picot2022adversarial}. To encompass both scenarios, and in particular to capture data-dependent regularization schemes, we adopt the more general notation $\mathcal{R}_{\mathcal{S}}(\theta)$ throughout this work.

In this study, we focus on divergence-based regularization, and in particular on the family of $f$-divergences, denoted by $D_{\phi}$. An $f$-divergence between two distributions $P$ and $Q$, with densities $p$ and $q$, respectively, is defined as
\begin{align}
    D_{\phi}(P \,\|\, Q) \triangleq \int q(y)\,\phi\!\left(\frac{p(y)}{q(y)}\right)\,dy,
\end{align}
where $\phi : \mathbb{R}_+ \to \mathbb{R}$ is a convex function satisfying $\phi(1) = 0$.

We consider regularization terms of the form
\begin{align}
    \mathcal{R}_{\mathcal{S}}(\theta) = D_{\phi}(P \,\|\, Q),
\end{align}
which gives rise to the divergence-based training objective 
\begin{align}
\label{eq:div_obj}
    L_{\mathcal S}^{\mathrm{Div}}(\theta)
    \triangleq
    L_{\mathcal S}(\theta)+\lambda\,D_\phi(P\|Q),
\end{align}
where $P$ denotes the output distribution induced by the model $f_\theta$, and $Q$ denotes a perturbed output distribution. Such perturbations may arise from input transformations, in which case $Q = f_\theta(T(x))$, from parameter perturbations, in which case $Q = f_{\theta+\delta}(x)$, or, more generally, from a combination of both, yielding $Q = f_{\theta+\delta}(T(x))$.

\begin{remark}
This regularization applies to predictive models whose outputs can be interpreted as probability distributions over the target space. This includes standard classification models, as well as probabilistic regression models.
\end{remark}

To better understand the effect of these regularization schemes, we analyze the local behavior of $f$-divergences when the two distributions $P$ and $Q$ are close. This regime is of particular interest in our setting, as the perturbations considered, whether induced in input space or parameter space, typically generate only small deviations in the model outputs. Studying this local geometry allows us to characterize how divergence-based regularization shapes the optimization landscape and influences robustness.

\subsection{Local Quadratic Geometry of $f$-Divergences}
We now analyze the local behavior of $f$-divergences in the regime where the two distributions $P$ and $Q$ are close. To formalize this, we consider a reference distribution $P$ with density $p$, and a perturbed distribution $Q$ with density $q = p + \varepsilon$, where $\varepsilon$ is a small perturbation satisfying $\int \varepsilon(x)\,dx = 0$. Under this assumption, we study the second-order expansion of the divergence $D_{\phi}(P \,\|\, Q)$.

\begin{proposition}[Local quadratic approximation of $f$-divergences]
\label{proposition:Local_approx_of_f-div}
Let $P$ and $Q$ be two distributions with densities $p$ and $q$, and assume that $q = p + \varepsilon$, where $\varepsilon$ is sufficiently small and satisfies $\int \varepsilon(y)\,dy = 0$. Then the $f$-divergence admits the second-order expansion
\begin{align}
    D_{\phi}(P \,\|\, Q)
    =
    \frac{\phi''(1)}{2} \int \frac{\varepsilon(y)^2}{p(y)}\,dy
    + o(\|\varepsilon\|^2).
\end{align}
\end{proposition}

This local quadratic approximation is classical; in an equivalent notation, it appears as a second-order expansion of $f$-divergences into a scaled chi-square term (see Corollary 1 in~\citep{nielsen2013chi}).

Proposition~\ref{proposition:Local_approx_of_f-div} shows that, in the local regime where $P$ and $Q$ are close, any $f$-divergence reduces to a quadratic penalty on the perturbation of the predictive distribution. Up to the multiplicative constant $\phi''(1)$, all $f$-divergences therefore share the same second-order behavior. In particular, the divergence no longer depends on the full nonlinear form of $\phi$, but only on the magnitude of the change in the output distribution under small perturbations.

The following corollaries make explicit how the second-order expansion of an $f$-divergence can be written when the perturbation arises either in parameter space or in input space. These results will be used later to relate divergence-based regularization to sharpness-aware objectives.

\begin{corollary}[Fisher form in parameter space]
\label{corollary:fisher-form}
Let $f_\theta : \mathcal X \to \mathcal P(\mathcal Y)$ be differentiable with respect to $\theta$ at a given parameter value $\theta$, and fix an input $x \in \mathcal X$. Define
\[
P = f_\theta(x), \qquad Q = f_{\theta+\delta}(x),
\]
where $\delta \in \mathbb R^d$ is a sufficiently small parameter perturbation. Assume that $P$ and $Q$ admit densities $p$ and $q$ with respect to a common reference measure on $\mathcal Y$, and that $\{f_\theta(x)\}_{\theta \in \Theta}$ forms a smooth parametric family with Fisher information matrix $F_x(\theta)$ evaluated at $x$. Then
\begin{align}
D_\phi(P \,\|\, Q)
=
\frac{\phi''(1)}{2}\,
\delta^\top F_x(\theta)\,\delta
+ o(\|\delta\|^2).
\end{align}
\end{corollary}

\begin{corollary}[Pullback form in input space]
\label{corollary:pullback-form}
Let $f_\theta : \mathcal X \to \mathcal P(\mathcal Y)$ be differentiable with respect to the input at a given point $x \in \mathcal X$, and let $T(x)=x+\Delta x$ be a sufficiently small input perturbation. Define
\[
P = f_\theta(x), \qquad Q = f_\theta(T(x)).
\]
Assume that $P$ and $Q$ admit densities $p$ and $q$ with respect to a common reference measure on $\mathcal Y$, and that the differential of the density satisfies
\[
d_x p(y)[\Delta x] = \nabla_x p(y)^\top \Delta x.
\]
Define
\[
\Gamma_x(\theta)
= 
\int
\frac{\bigl(\nabla_x p(y)\bigr)\bigl(\nabla_x p(y)\bigr)^\top}{p(y)}\,dy.
\]
Then
\begin{align}
    D_\phi(P \,\|\, Q)
=
\frac{\phi''(1)}{2}\,
\Delta x^\top \Gamma_x(\theta)\,\Delta x
+ o(\|\Delta x\|^2).
\end{align}
\end{corollary}
Both corollaries follow by applying Proposition~\ref{proposition:Local_approx_of_f-div} to the perturbation induced, respectively, by the parameter differential and the input differential, and then identifying the resulting quadratic form with the Fisher information matrix in parameter space and with its pullback analogue in input space. Detailed proofs are provided in Appendix~\ref{app:fisher-form} and Appendix~\ref{app:pullback-form}, respectively. 

In the parameter-perturbation setting, the divergence penalizes local variations of the predictive distribution induced by perturbations of $\theta$, yielding the Fisher-information form of Corollary~\ref{corollary:fisher-form}. In the input-perturbation setting, it penalizes local variations induced by perturbations of $x$, leading to the pullback metric of Corollary~\ref{corollary:pullback-form}. Taken together, these results show that divergence-based regularization acts as a local sensitivity penalty: it favors models whose predictive distributions remain stable under small perturbations, whether applied to the parameters or to the inputs. This interpretation will be central in the sequel, where we show that, in the parameter-perturbation setting, divergence-based regularization induces a local objective closely related to sharpness-aware minimization.

\subsection{Example: Curvature Analysis in the Asymmetric $\alpha$-Skew JSD Family}

Building on the local quadratic analysis above, we now consider a subfamily of $f$-divergences, namely the asymmetric $\alpha$-skew JSD family introduced in~\citep{nielsen2020generalization}. This family provides a convenient one-parameter interpolation within a fixed class of divergences, allowing us to compare how the curvature induced by the regularization varies with $\alpha$. Although the divergence itself changes with $\alpha$, it does so within a single parametric family, thereby enabling a controlled comparison of local curvature effects across different degrees of asymmetry.

The asymmetric $\alpha$-skew JSD is defined as
\begin{align}
\mathrm{JSD}^\alpha_a(P \,\|\, Q)
&\triangleq (1-\alpha)\, D_{\mathrm{KL}}\!\bigl(P \,\big\|\, (1 - \alpha) P + \alpha Q\bigr)
+ \alpha\, D_{\mathrm{KL}}\!\bigl(Q \,\big\|\, (1 - \alpha) P + \alpha Q\bigr),
\quad \alpha \in (0,1).
\end{align}

It satisfies the identity
\begin{align}
    \mathrm{JSD}^\alpha_a(P \,\|\, Q) = \mathrm{JSD}^{1-\alpha}_a(Q \,\|\, P),
\end{align}

which implies that $\mathrm{JSD}^\alpha_a$ is symmetric in $(P,Q)$ only for $\alpha=\tfrac{1}{2}$, in which case it reduces to the standard JSD. Moreover, the divergence approaches $0$ in the limits $\alpha \to 0$ and $\alpha \to 1$.

As $\alpha$ varies from $0$ to $1$, $\mathrm{JSD}^\alpha_a$ continuously shifts from emphasizing
$D_{\mathrm{KL}}\!\bigl(P \,\|\, (1-\alpha)P + \alpha Q\bigr)$
to emphasizing
$D_{\mathrm{KL}}\!\bigl(Q \,\|\, (1-\alpha)P + \alpha Q\bigr)$ (see Appendix~\ref{app:alphaJSD}). 
This makes the family particularly suitable for analyzing how the curvature of the induced regularization evolves under a controlled change in asymmetry.

We now make this dependence explicit by deriving how the local second-order coefficient depends on $\alpha$ within this family.

\begin{proposition}[Curvature of the asymmetric $\alpha$-skew JSD family]
\label{proposition:curv_asymm_jsd}
Let $P$ and $Q$ be two distributions with densities $p$ and $q$, and assume that
$$
q(y) = p(y) + \varepsilon(y),
$$
where $\varepsilon$ is sufficiently small and satisfies
$$
\int \varepsilon(y)\,dy = 0.
$$
Then, the asymmetric $\alpha$-skew JSD admits the second-order expansion
\begin{align}
    \mathrm{JSD}_a^\alpha(P \,\|\, Q)
    =
    \frac{\alpha(1-\alpha)}{2}
    \int \frac{\varepsilon(y)^2}{p(y)}\,dy
    + o(\|\varepsilon\|^2).
\end{align}
In particular, the local curvature coefficient is proportional to $\alpha(1-\alpha)$, and is therefore maximized at $\alpha=\tfrac{1}{2}$.
\end{proposition}

The proof is provided in Appendix~\ref{app:curv_alpha_JSD}. It can be obtained by specializing the local quadratic expansion of Proposition~\ref{proposition:Local_approx_of_f-div} to the asymmetric $\alpha$-skew JSD family; for completeness, we provide a direct derivation based on Padé approximations and Taylor expansions.

Proposition~\ref{proposition:curv_asymm_jsd} shows that, within the asymmetric $\alpha$-skew JSD family, varying $\alpha$ does not merely change the asymmetry of the divergence: it also rescales the strength of the induced local quadratic penalty through the factor $\alpha(1-\alpha)$. Since this factor is maximal at $\alpha = \tfrac12$, the symmetric JSD yields the greatest local curvature in this family.

\section{Local Connection to Sharpness-Aware Minimization}
\label{section:local_connection_to_sam}
The previous analysis shows that $f$-divergence-based regularization induces a local quadratic penalty in parameter space. We now use this characterization to establish a local second-order comparison with Sharpness-Aware Minimization (SAM). At a high level, both approaches promote robustness to parameter perturbations: divergence-based regularization enforces stability of the predictive distribution under small perturbations, whereas SAM explicitly favors parameter neighborhoods in which the training loss remains small.

Recall that SAM is defined through the min--max objective
\begin{align}
    \min_{\theta \in \Theta} \; \max_{\|\delta\|\leq \rho} \, L_{\mathcal{S}}(\theta+\delta),
\end{align}
where $\rho > 0$ controls the size of the adversarial neighborhood in parameter space. Equivalently, defining the SAM loss by
\begin{align}
    L_{\mathcal{S}}^{\mathrm{SAM}}(\theta)
    \triangleq
    \max_{\|\delta\|\leq \rho} \; L_{\mathcal{S}}(\theta+\delta),
\end{align}
the optimization problem becomes
\begin{align}
    \min_{\theta \in \Theta} \; L_{\mathcal{S}}^{\mathrm{SAM}}(\theta).
\end{align}

The inner maximization in SAM can be analyzed through a second-order Taylor expansion of the loss around $\theta$:
\begin{align}
    L_{\mathcal{S}}(\theta+\delta)
    =
    L_{\mathcal{S}}(\theta)
    + \nabla_\theta L_{\mathcal{S}}(\theta)^\top \delta
    + \frac{1}{2}\delta^\top H_{\mathcal{S}}(\theta)\delta
    + o(\|\delta\|^2),
\end{align}
where $H_{\mathcal{S}}(\theta)$ denotes the Hessian of the training loss. Substituting this expansion into the SAM objective yields
\begin{align}
    L_{\mathcal{S}}^{\mathrm{SAM}}(\theta)
    =
    \max_{\|\delta\|\leq \rho}
    \left[
        L_{\mathcal{S}}(\theta)
        + \nabla_\theta L_{\mathcal{S}}(\theta)^\top \delta
        + \frac{1}{2}\delta^\top H_{\mathcal{S}}(\theta)\delta
    \right]
    + o(\rho^2).
\end{align}

Near a stationary point, where $\nabla_\theta L_{\mathcal{S}}(\theta)\approx 0$, the linear term becomes negligible, and the inner maximization is locally dominated by the curvature term alone:
\begin{align}
    L_{\mathcal{S}}^{\mathrm{SAM}}(\theta)
    \approx
    L_{\mathcal{S}}(\theta)
    +
    \max_{\|\delta\|\leq \rho}
    \frac{1}{2}\delta^\top H_{\mathcal{S}}(\theta)\delta.
\end{align}

When the dominant curvature is nonnegative, as is the case in particular near a local minimum, the maximization of this quadratic form over the Euclidean ball can be characterized through the Rayleigh quotient. Indeed, for a symmetric matrix, the Rayleigh-Ritz theorem gives the largest eigenvalue as the maximum of the associated quadratic form over the unit sphere~\citep{horn2012matrix}. Applying this to $H_{\mathcal{S}}(\theta)$ yields
\begin{align}
    \max_{\|\delta\|\leq \rho}
    \frac{1}{2}\delta^\top H_{\mathcal{S}}(\theta)\delta
    =
    \frac{\rho^2}{2}\,\lambda_{\max}\!\bigl(H_{\mathcal{S}}(\theta)\bigr),
\end{align}
where the derivation is provided in Appendix~\ref{app:rayleigh_ball}. Consequently, in the regime where the dominant curvature is nonnegative and the iterate is near stationarity,
\begin{align}
\label{eq:sam_lambda}
    L_{\mathcal{S}}^{\mathrm{SAM}}(\theta)
    =
    L_{\mathcal{S}}(\theta)
    +
    \frac{\rho^2}{2}\,\lambda_{\max}\!\bigl(H_{\mathcal{S}}(\theta)\bigr)
    +
    o(\rho^2).
\end{align}

Thus, near stationarity and under nonnegative dominant curvature, SAM admits a local second-order interpretation as a spectral curvature penalty, consistent with prior analyses~\citep{agarwala2023sam} showing that it suppresses directions associated with large Hessian eigenvalues during training.

To compare the SAM objective in Eq. (\ref{eq:sam_lambda}) with the divergence-based objective in Eq. (\ref{eq:div_obj}), we consider the parameter perturbation case where $P = f_\theta(x)$ and $ Q = f_{\theta+\delta}(x).$
By Corollary~\ref{corollary:fisher-form}, the divergence-based objective admits the local expansion
\begin{align}
    L_{\mathcal S}^{\mathrm{Div}}(\theta)
    = L_{\mathcal{S}}(\theta) + \lambda
    \frac{\phi''(1)}{2}\,\delta^\top F_{\mathcal{S}}(\theta)\,\delta
    + o(\|\delta\|^2),
\end{align}
where $F_S(\theta) = \frac{1}{n}\sum_{i=1}^n F_{x_i}(\theta)$ denotes the Fisher information matrix associated with the predictive distributions over the training set $\mathcal{S}$. Hence, while SAM induces a spectral curvature penalty governed by the dominant Hessian eigenvalue, divergence-based regularization induces a Fisher-weighted quadratic penalty on parameter perturbations.

The comparison established above suggests that the key remaining question is whether the Fisher geometry induced by divergence-based regularization can be related to the Hessian geometry underlying SAM. In general, these two objects are distinct. However, as discussed by~\citet{martens2020new}, for many important probabilistic losses the Fisher matrix coincides with, or provides a natural approximation to, the generalized Gauss-Newton (GGN) matrix, which in turn captures the positive-semidefinite component of the Hessian.

In particular, for negative log-likelihood (NLL) objectives with exponential-family output distributions, the Fisher and GGN matrices coincide~\citep{martens2020new}. This includes the common case of softmax outputs trained with cross-entropy loss. In such settings, the curvature geometry induced by divergence-based regularization through the Fisher matrix is therefore closely aligned with the curvature structure underlying second-order interpretations of SAM.

Consequently, although divergence-based regularization and SAM are not identical objectives, they become locally comparable through a common curvature geometry: divergence-based regularization penalizes parameter perturbations through a Fisher-weighted quadratic form, while SAM penalizes sharp directions of the loss landscape through the Hessian. When the Hessian is well approximated by its Gauss-Newton component, these two perspectives become locally aligned.

While the direct connection to SAM arises under the parameter perturbation setting, the input-space analysis highlights a broader advantage of divergence-based regularization. Indeed, the same framework naturally accommodates perturbations not only in parameter space, but also in input space, or even in combined input-parameter space. In this sense, divergence-based regularization operates over a more general perturbation regime than standard SAM, which is defined solely through adversarial perturbations of the parameters. Corollary~\ref{corollary:pullback-form} shows that, in the input-space setting, the resulting objective penalizes directions along which the predictive distribution varies sharply under small input perturbations, thereby connecting divergence-based regularization to local input robustness as well as parameter-space sharpness control.

\section{Empirical Validation}
\label{section:experiments}

The theoretical analysis developed in Sections~\ref{section:general_framework} and~\ref{section:local_connection_to_sam} provides a general local interpretation of divergence-based regularization: when the predictive distributions before and after perturbation are close, the regularization term reduces to a second-order penalty whose geometry depends on the chosen divergence. To test whether this local curvature has an observable empirical effect, we require a family of divergences in which the strength of the induced curvature can be varied in a controlled manner without otherwise changing the overall structure of the regularizer. The asymmetric $\alpha$-skew JSD family is particularly well suited for this purpose. As shown in Proposition~\ref{proposition:curv_asymm_jsd}, its local curvature coefficient scales as $\alpha(1-\alpha)$, yielding a one-parameter family in which the curvature varies smoothly and reaches its maximum at the symmetric point $\alpha=\tfrac12$. This makes it a natural testbed for a controlled empirical validation of the theoretical predictions.

\subsection{Experimental setup}
\paragraph{Datasets.}
We employed four benchmark datasets in our experiments: CIFAR-10, Fashion-MNIST, EMNIST (Balanced), and Oxford-IIIT Pet. These datasets were chosen for their widespread adoption as standard benchmarks and for the diversity of their characteristics, including image resolution, number of classes, dataset size, and color modality. This diversity allows us to assess whether the empirical trends associated with the regularization are stable across a range of learning settings.

\paragraph{Perturbation mechanism.}
The experiments are conducted in the input-perturbation setting. We implement the transformation $T$ as random masking of a fixed proportion of input features, and compare the predictive distributions on the original input $x$ and its perturbed version $T(x)$. To approximately probe the local perturbative regime motivated by the theoretical analysis, we use a small masking intensity and set the proportion of masked features to $2\%$. This yields a mild input perturbation that is sufficient to probe local sensitivity while preserving the overall content of the example. A detailed description of the transformation mechanism is provided in Appendix~\ref{app:transformation}.

\paragraph{Model and training.}
We employed the pretrained EfficientNet-B2 model~\citep{tan2019efficientnet} and fine-tuned it on each dataset for two epochs, except for the Oxford-IIIT Pet dataset, which required thirteen epochs to ensure convergence due to its smaller size and higher class granularity. Training was performed using the AdamW optimizer with a learning rate of $10^{-4}$, a batch size of $32$, and nine independent random seeds per dataset.

\paragraph{Metrics.}
Model performance is evaluated using two complementary metrics: accuracy and NLL. While accuracy measures predictive correctness, NLL assesses the quality of the model's probabilistic predictions.

\paragraph{Controlled variation of the curvature coefficient.}

To isolate the effect of the theoretically predicted curvature variation, we vary only the asymmetry parameter $\alpha$ within the asymmetric $\alpha$-skew JSD family, considering values $\alpha \in \{0.1,0.2,\dots,0.9\}$, while keeping all other components of the training pipeline fixed. Since Proposition~\ref{proposition:curv_asymm_jsd} shows that the local second-order coefficient is given by $\alpha(1-\alpha)$, this design yields a controlled experimental setting in which changes in performance can be directly compared with the predicted variation in curvature.

\subsection{Empirical performance across the asymmetric $\alpha$-skew JSD family}

\begin{figure*}[h!]
    \centering
    \includegraphics[width=\linewidth]{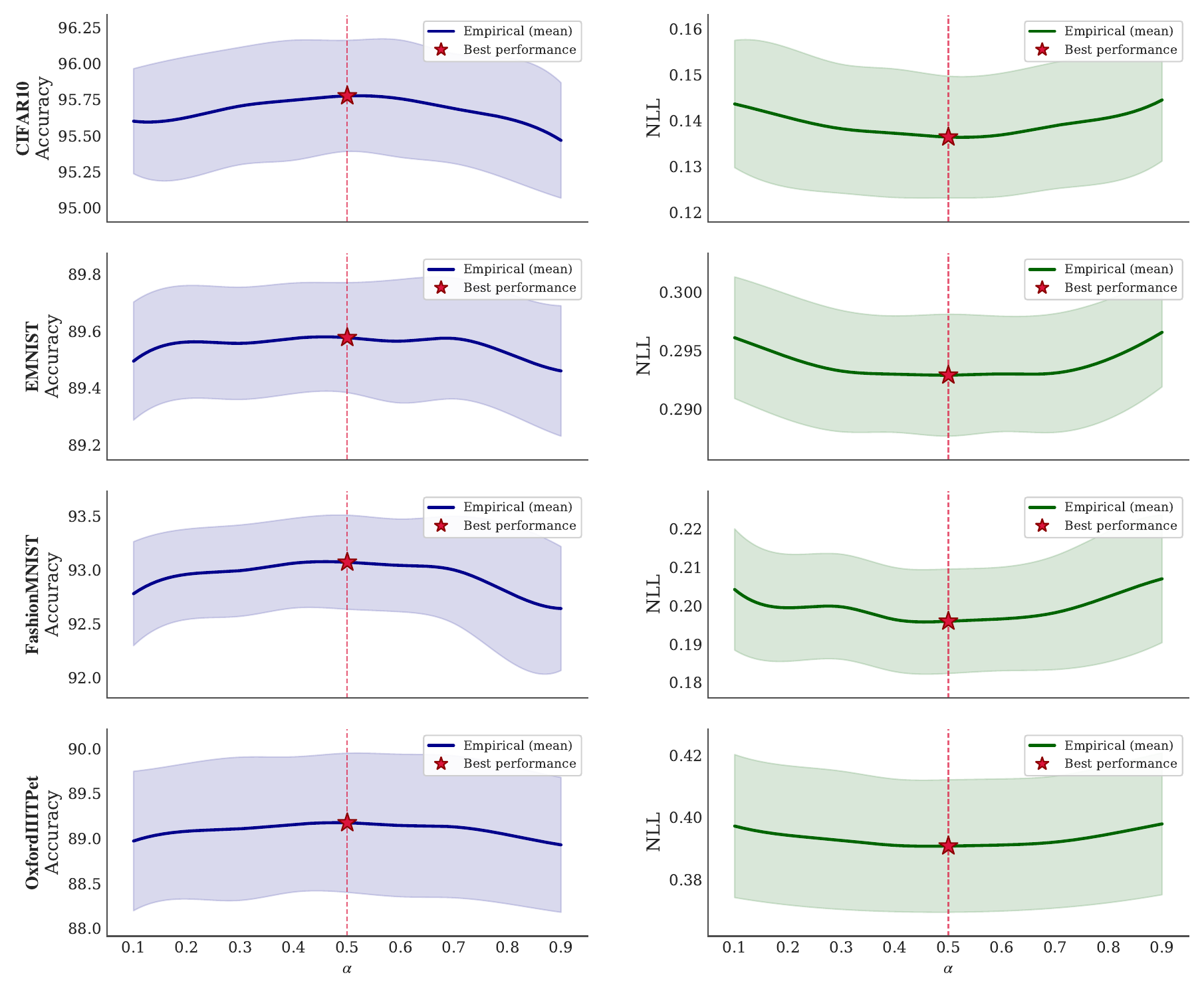}
    \caption{Model performance as a function of the asymmetry parameter $\alpha$ across the evaluated datasets.}
    \label{fig:results}
\end{figure*}

Figure~\ref{fig:results} summarizes the impact of the asymmetric $\alpha$-skew JSD regularization on model performance, in terms of accuracy and NLL, across the evaluated datasets. Since the theory predicts that the induced local curvature scales as $\alpha(1-\alpha)$, attaining its maximum at $\alpha=\tfrac12$, this experiment provides a controlled test of whether a stronger local curvature-penalization coefficient is associated with improved empirical behavior.

Across all datasets, performance is consistently strongest near the symmetric regime, where the theoretical curvature coefficient is largest. In particular, accuracy tends to increase and NLL tends to decrease as $\alpha$ approaches $\tfrac12$, indicating enhanced predictive performance together with improved probabilistic fit. The best overall results are achieved in the symmetric case $\alpha=\tfrac12$, which is precisely the point at which the theoretical local curvature reaches its maximum.

To further characterize the shape of the empirical response as a function of $\alpha$, we also fit symmetric power-law models around $\alpha=\tfrac12$; the corresponding analysis is reported in Appendix~\ref{app:powerlaw_fits}.

Overall, the empirical trends are consistent with the theoretical analysis. Within the JSD family, varying $\alpha$ changes the local second-order coefficient without otherwise modifying the overall structure of the regularizer. The fact that performance is strongest near $\alpha=\tfrac12$ therefore supports the interpretation that the strength of the induced local curvature plays a meaningful role in the effectiveness of divergence-based regularization within this family. 

\subsection{Loss Landscape Comparison}
\begin{figure}[h!]
     \centering
     \includegraphics[width=\linewidth]{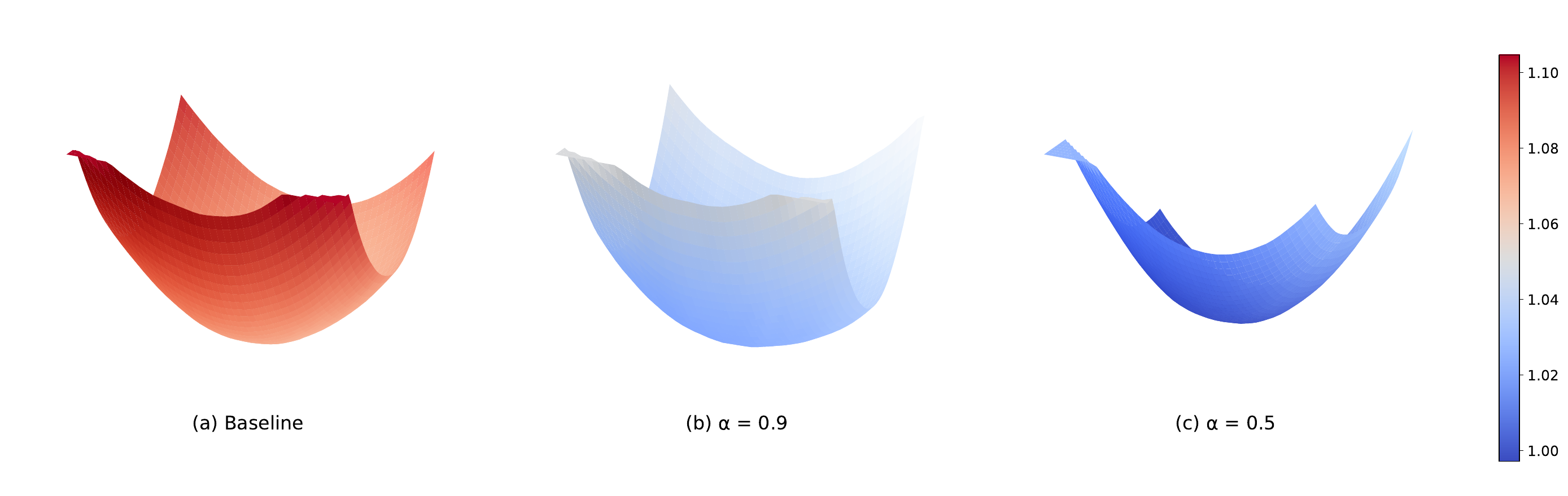}
     \caption{Three-dimensional renderings of the loss landscapes of models trained without divergence regularization (baseline), with asymmetric $\alpha$-skew JSD regularization with $\alpha = 0.9$, and with $\alpha = 0.5$, evaluated on a common two-dimensional perturbation plane around the final parameters. Colors denote loss values on a shared scale.}
    \label{fig:landscape-3d}
 \end{figure}

To perform the visualizations in Figure~\ref{fig:landscape-3d}, we train a ResNet-18 on CIFAR-10 for $80$ epochs and then visualize the loss landscape using the procedure described in~\cite{li2018visualizing}. The model is trained without pretraining, with the final classification layer adapted to the $10$ classes. Optimization is performed using SGD with momentum of $0.9$, an initial learning rate of $0.1$, weight decay of $5 \times 10^{-4}$, and a batch size of $256$. A multi-step learning rate schedule decreases the learning rate by a factor of $0.1$ at epochs $30$ and $45$. The baseline is trained with the cross-entropy loss alone. We compare it to models trained with the same setup but augmented with the asymmetric $\alpha$-skew JSD regularizer. 

Figure~\ref{fig:landscape-3d} illustrates that incorporating the asymmetric $\alpha$-skew JSD regularization drives the optimization toward flatter local minima. Relative to the baseline, the model trained with $\alpha=0.9$ exhibits a flatter loss landscape. Furthermore, the comparison between $\alpha=0.9$ and $\alpha=0.5$ indicates that the choice of $\alpha$ also affects the geometry of the minimum, with the more symmetric setting $\alpha=0.5$, corresponding to the strongest local curvature and thus the strongest penalization, producing the flattest basin. For completeness, the two-dimensional visualization and analysis are provided in Appendix~\ref{app:landscape-2d}.

\begin{table}[t]
\centering
\caption{Average loss-landscape statistics across three seeds for each training method.}
\label{tab:landscape_stats}
\begin{tabular}{lccc}
\toprule
Method & Avg. loss & Avg. $\mathrm{tr}(H_{2D})$ & Avg. $\lambda_{\max}(H_{2D})$ \\
\midrule
Baseline & $1.012 \pm 0.056$ & $6.426 \pm 3.591$ & $3.928 \pm 2.384$ \\
$\alpha = 0.9$ & $1.019 \pm 0.017$ & $4.745 \pm 3.260$ & $2.945 \pm 1.852$ \\
$\alpha = 0.5$ & $\bm{1.003 \pm 0.033}$ & $\bm{3.508 \pm 2.689}$ & $\bm{2.700 \pm 1.512}$ \\
$\alpha = 0.1$ & $1.027 \pm 0.013$ & $4.622 \pm 1.532$ & $2.959 \pm 1.242$ \\
\bottomrule
\end{tabular}
\end{table}

Table~\ref{tab:landscape_stats} quantitatively supports the visual trends in Figure~\ref{fig:landscape-3d}. Averaged across three random seeds, divergence-based regularization reduces both the trace and the maximum eigenvalue of the two-dimensional Hessian approximation relative to the baseline, indicating a flatter local loss landscape. Among the evaluated settings, $\alpha=0.5$ yields the lowest average loss together with the smallest curvature measures, which is consistent with the qualitative observation that the symmetric regime produces the flattest basin. Although the variability across seeds is non-negligible, the overall trend remains consistent: stronger curvature penalization is associated with reduced local sharpness.

\section{Limitations}
\label{section:limitations}
The fact that all $f$-divergences reduce, up to second order, to the same quadratic form scaled by $\phi''(1)$ shows that divergence-based regularizers share a common local geometric interpretation, which is naturally linked to sharpness-aware minimization. At the same time, this universality reveals a limitation of second-order analysis: since all $f$-divergences coincide locally up to a scaling factor, such an analysis alone cannot determine which divergence family should be preferred in practice. This suggests that meaningful differences between divergence families must arise beyond the quadratic regime, for instance through higher-order terms in the local expansion or through structural properties such as symmetry and boundedness. Developing such criteria for divergence selection is therefore a natural direction for future work.

\section{Conclusion}
In this paper, we studied $f$-divergence-based regularization through the lens of local second-order geometry and established a formal connection with sharpness-aware minimization. Leveraging classical results on the local quadratic behavior of $f$-divergences, we showed that divergence-based regularization induces a local curvature geometry that becomes comparable to second-order interpretations of SAM in parameter space under standard probabilistic-loss assumptions. Beyond the parameter perturbation setting, divergence-based regularization naturally extends to input perturbations, thereby defining a more general perturbation regime than SAM. 

The asymmetric $\alpha$-skew JSD family provides a controlled example of how variations in local curvature within a fixed subfamily of $f$-divergences affect model performance. We showed that its local second-order coefficient scales as $\alpha(1-\alpha)$ and is maximized at the symmetric point $\alpha=\tfrac12$. In turn, stronger local coefficients within the $\alpha$-skew JSD family were empirically associated with improved performance and flatter projected loss-landscape profiles in our input-perturbation experiments.

\section*{Acknowledgments}
This work was supported by the IMT “Futur, Ruptures \& Impacts” programme, by the European Research Council (ERC) under the European Union’s Horizon 2020 research and innovation programme (Grant Agreement No. 101003431, SONATA), and by the Smart Networks and Services Joint Undertaking (SNS JU) under the European Union’s Horizon Europe research and innovation programme (Grant Agreement No. 101192080, 6G-LEADER).

\bibliography{collas2026_conference}
\bibliographystyle{collas2026_conference}

\newpage
\appendix
\section{Proofs}
\subsection{Fisher form}
\label{app:fisher-form}
\begin{proof}[Proof of Corollary~\ref{corollary:fisher-form}]
Let $P=f_\theta(x)$ and $Q=f_{\theta+\delta}(x)$, and denote by $p$ and $q$ their respective densities with respect to a common reference measure on $\mathcal Y$. Since $\theta \mapsto f_\theta(x)$ is differentiable, we have the first-order expansion
\[
q(y)-p(y)=d_\theta p(y)[\delta]+o(\|\delta\|).
\]
Applying Proposition~\ref{proposition:Local_approx_of_f-div} with $\varepsilon(y)=q(y)-p(y)$ gives
\[
D_\phi(P\|Q)
=
\frac{\phi''(1)}{2}
\int
\frac{(d_\theta p(y)[\delta])^2}{p(y)}\,dy
+o(\|\delta\|^2).
\]
Expressing the differential in coordinates, we obtain
\[
d_\theta p(y)[\delta]
=
\nabla_\theta p(y)^\top \delta.
\]
Then
\[
\int
\frac{(d_\theta p(y)[\delta])^2}{p(y)}\,dy
=
\int
\frac{\delta^\top \nabla_\theta p(y)\nabla_\theta p(y)^\top \delta}{p(y)}\,dy
=
\delta^\top
\left(
\int
\frac{\nabla_\theta p(y)\nabla_\theta p(y)^\top}{p(y)}\,dy
\right)
\delta.
\]
By definition, the matrix in parentheses is the Fisher information matrix $F_x(\theta)$. Therefore,
\[
D_\phi(P\|Q)
=
\frac{\phi''(1)}{2}\,
\delta^\top F_x(\theta)\,\delta
+o(\|\delta\|^2).
\]
This completes the proof.
\end{proof}

\subsection{Pullback form}
\label{app:pullback-form}

\begin{proof}[Proof of Corollary~\ref{corollary:pullback-form}]
Let $P=f_\theta(x)$ and $Q=f_\theta(T(x))$, where $T(x)=x+\Delta x$, and denote by $p$ and $q$ their respective densities with respect to a common reference measure on $\mathcal Y$. Since $x \mapsto f_\theta(x)$ is differentiable, we have
\[
q(y)-p(y)=d_x p(y)[\Delta x]+o(\|\Delta x\|).
\]
Applying Proposition~\ref{proposition:Local_approx_of_f-div} with $\varepsilon(y)=q(y)-p(y)$ yields
\[
D_\phi(P\|Q)
=
\frac{\phi''(1)}{2}
\int
\frac{(d_x p(y)[\Delta x])^2}{p(y)}\,dy
+o(\|\Delta x\|^2).
\]
Using the representation
\[
d_x p(y)[\Delta x]=\nabla_x p(y)^\top \Delta x,
\]
we obtain
\[
\int
\frac{(d_x p(y)[\Delta x])^2}{p(y)}\,dy
=
\int
\frac{\Delta x^\top \nabla_x p(y)\nabla_x p(y)^\top \Delta x}{p(y)}\,dy
=
\Delta x^\top
\left(
\int
\frac{\nabla_x p(y)\nabla_x p(y)^\top}{p(y)}\,dy
\right)
\Delta x.
\]
Defining
\[
\Gamma_x(\theta)
= 
\int
\frac{\nabla_x p(y)\nabla_x p(y)^\top}{p(y)}\,dy,
\]
we conclude that
\[
D_\phi(P\|Q)
=
\frac{\phi''(1)}{2}\,
\Delta x^\top \Gamma_x(\theta)\,\Delta x
+o(\|\Delta x\|^2).
\]
This completes the proof.
\end{proof}

\subsection{Curvature of the Asymmetric $\alpha$-skew JSD}
\label{app:curv_alpha_JSD}
\begin{proof}[Proof of Proposition~\ref{proposition:curv_asymm_jsd}]
For completeness, we restate Proposition~\ref{proposition:curv_asymm_jsd}.  

Let $P$ and $Q$ be two distributions with densities $p$ and $q$, and assume that
$$
q(y) = p(y) + \varepsilon(y),
$$
where $\varepsilon$ is sufficiently small and satisfies
$$
\int \varepsilon(y)\,dy = 0.
$$
Then, the asymmetric $\alpha$-skew JSD admits the second-order expansion
$$
    \mathrm{JSD}_a^\alpha(P \,\|\, Q)
    =
    \frac{\alpha(1-\alpha)}{2}
    \int \frac{\varepsilon(y)^2}{p(y)}\,dy
    + o(\|\varepsilon\|^2).
$$

To prove the result, we employ the $[1/1]$ Padé approximant of $\log(1+u)$: 
\begin{align}
   \label{eq:pade}
   \log(1+u) \approx \frac{2u}{2+u},  
\end{align}
which matches the Taylor expansion up to second order while providing a more stable approximation.

Let $M_{\alpha} = (1-\alpha) P + \alpha Q = P + \alpha \varepsilon$. We decompose the proof into three steps
\begin{itemize}
    \item The expansion of $D_{\mathrm{KL}}(P||M_\alpha)$.
    \item The expansion of $D_{\mathrm{KL}}(Q||M_\alpha)$.
    \item The expansion of $(1-\alpha)\, D_{\mathrm{KL}}\!\bigl(P \,\big\|\, M_\alpha\bigr)
    + \alpha\, D_{\mathrm{KL}}\!\bigl(Q \,\big\|\, M_\alpha\bigr)$ obtained by combining the previous two expansions.
\end{itemize}

\paragraph{Expansion of $D_{\mathrm{KL}}(P||M_\alpha)$.}
$$
D_{\mathrm{KL}}(P||M_\alpha) = 
 \int p(y)\log\frac{p(y)}{p(y)+\alpha \varepsilon(y)}dy
= -\int p(y)\log\left(1 + \alpha\frac{\varepsilon(y)}{p(y)}\right)dy.
$$
Using the Padé approximation (\eqref{eq:pade}), we obtain
$$\log\left(1 + \alpha\frac{\varepsilon}{p}\right)
\approx
\frac{2 \alpha \frac{\varepsilon(y)}{p(y)}}{2 + \alpha \frac{\varepsilon(y)}{p(y)}}.
$$
Thus:
$$
D_{\mathrm{KL}}(P||M_\alpha)
\approx
-\int p\left(
\frac{2 \alpha \frac{\varepsilon}{p}}{2 + \alpha \frac{\varepsilon}{p}}
\right)  
= - \int \alpha \varepsilon \left(1+ \alpha \frac{\varepsilon}{2p}\right)^{-1}. 
$$
Using the Taylor expansion of $(1+u)^{-1}$ when $u \to 0$, we obtain: 
$$    \left(1+ \alpha \frac{\varepsilon}{2p}\right)^{-1} = 1 - \frac{\alpha \varepsilon}{2p} + \alpha^2 \frac{\varepsilon^2}{4p^2} + o(\varepsilon^2). 
$$
Thus, up to the second order: 
$$    D_{\mathrm{KL}}(P||M_\alpha) \approx -\alpha \int \varepsilon(y) dy + \frac{\alpha^2}{2} \int \frac{\varepsilon^2(y)}{p(y)} dy. 
$$

Using $\int \varepsilon = 0$, we obtain:
\begin{align}
\label{eq:kl_pm}
D_{\mathrm{KL}}(P||M_\alpha)
\approx
\frac{\alpha^2}{2}\int \frac{\varepsilon^2}{p}.
\end{align}

\paragraph{Expansion of $D_{\mathrm{KL}}(Q||M_\alpha)$.}
We write:
$$
D_{\mathrm{KL}}(Q||M_\alpha) = 
\int (p(y)+\varepsilon(y))\log\frac{p(y)+\varepsilon(y)}{p(y)+\alpha\varepsilon(y)}dy.
$$
The ratio can be expressed as:
$$
\frac{p+\varepsilon}{p+\alpha\varepsilon} = 1 + 
\frac{(1 - \alpha) \varepsilon}{p + \alpha \varepsilon}.
$$
Let $$v = \frac{(1 - \alpha) \varepsilon}{p + \alpha \varepsilon}.
$$ 
We have: 
$$
\frac{1}{p+\alpha \varepsilon} = \frac{1}{p} (1 + \alpha \frac{\varepsilon}{p})^{-1}
$$
Using the Taylor expansion of $(1+u)^{-1}$ as $u \to 0$: 
$$
(1 + \alpha \frac{\varepsilon}{p})^{-1} \approx 1 - \alpha \frac{\varepsilon}{p} + \alpha^2 \frac{\varepsilon^2}{p^2}
$$
Then: 
\begin{align}
\label{eq:v}
  v = (1-\alpha)\frac{\varepsilon}{p} - \alpha (1 -\alpha) \frac{\varepsilon^2}{p^2} + O(\varepsilon^3).   
\end{align}
Using the Padé approximation (\eqref{eq:pade}) then the Taylor expansion of $(1+u)^{-1}$, we obtain:
\begin{align}
\label{eq:log_1+v}
    \log(1+v) =  v(1 - \frac{v}{2} + o(v^2)).
\end{align}
Substituting the expression for $v$ from \eqref{eq:v} into \eqref{eq:log_1+v} and expanding, we obtain: 
$$
\log(1+v) = (1-\alpha)\frac{\varepsilon}{p} - \frac{1}{2}(1-\alpha)(1+\alpha)\frac{\varepsilon^2}{p^2} + o(\varepsilon^2).
$$
Multiplying by $(p+\varepsilon)$ and keeping terms up to second order:
\begin{align*}
    (p+\varepsilon) \log \frac{p+\varepsilon}{p+ \alpha \varepsilon} &= (p+\varepsilon) \left((1-\alpha)\frac{\varepsilon}{p} - \frac{1}{2}(1-\alpha)(1+\alpha)\frac{\varepsilon^2}{p^2} + o(\varepsilon^2) \right) 
\\ &= (1 - \alpha)\varepsilon - \frac{1}{2} (1-\alpha)(1+\alpha) \frac{\varepsilon^2}{p} + (1-\alpha)\frac{\varepsilon^2}{p} + o(\varepsilon^2)
\\ &\approx (1-\alpha)\varepsilon + \frac{1}{2} (1-\alpha)^2 \frac{\varepsilon^2}{p}.
\end{align*}
Integrating and using $\int \varepsilon = 0$, we obtain:
\begin{align}
\label{eq:kl_qm}
D_{\mathrm{KL}}(Q||M_\alpha)
\approx
\frac{(1-\alpha)^2}{2}\int \frac{\varepsilon^2}{p}.
\end{align}

\paragraph{Combination.}
The asymmetric $\alpha$-skew JSD is given by
$$
JSD_a^{\alpha}(P||Q) =  (1-\alpha)D_{\mathrm{KL}}(P||M_\alpha)
+
\alpha D_{\mathrm{KL}}(Q||M_\alpha).
$$

Substituting the expansions obtained in \eqref{eq:kl_pm} and \eqref{eq:kl_qm}, we obtain
$$
JSD_a^{\alpha}(P||Q) 
\approx
(1-\alpha)\frac{\alpha^2}{2} \int \frac{\varepsilon^2}{p}
+
\alpha\frac{(1-\alpha)^2}{2}\int \frac{\varepsilon^2}{p}. 
$$
Factoring the common terms yields
$$
JSD_a^{\alpha}(P||Q) \approx 
\frac{1}{2} \alpha (1 - \alpha) \int \frac{\varepsilon^2}{p}.
$$
This completes the proof.
\end{proof}

\subsection{Quadratic maximization over the Euclidean ball}
\label{app:rayleigh_ball}
\begin{proof}
We start by recalling how the maximization of a quadratic form over a Euclidean ball follows from the Rayleigh quotient characterization on the unit sphere. Let $A$ be a symmetric matrix, and consider
$$
\max_{\|x\|_2 \leq \rho} x^\top A x.
$$
Any vector $x$ in the ball can be written as
$$
x = r u,
\qquad 0 \leq r \leq \rho,
\qquad \|u\|_2 = 1.
$$
Substituting this decomposition into the quadratic form gives
$$
x^\top A x = r^2\, u^\top A u.
$$
For a fixed direction $u$, the dependence on $r$ is entirely through the factor $r^2$. Therefore, if $u^\top A u > 0$, the maximum along that ray is attained at the boundary $r=\rho$, whereas if $u^\top A u < 0$, the maximum is attained at $r=0$. It follows that
$$
\max_{\|x\|_2 \leq \rho} x^\top A x
=
\max\!\left(0,\ \rho^2 \max_{\|u\|_2 = 1} u^\top A u\right).
$$
Since $A$ is symmetric, the Rayleigh quotient theorem yields
$$
\max_{\|u\|_2 = 1} u^\top A u = \lambda_{\max}(A).
$$
Hence,
$$
\max_{\|x\|_2 \leq \rho} x^\top A x
=
\rho^2 \max\!\bigl(\lambda_{\max}(A),0\bigr).
$$
Equivalently,
$$
\max_{\|x\|_2 \leq \rho} \frac{1}{2}x^\top A x
=
\frac{\rho^2}{2}\max\!\bigl(\lambda_{\max}(A),0\bigr).
$$
In particular, when $\lambda_{\max}(A)\geq 0$, the maximum is attained on the boundary of the ball along an eigenvector associated with the largest eigenvalue, and one recovers
$$
\max_{\|x\|_2 \leq \rho} \frac{1}{2}x^\top A x
=
\frac{\rho^2}{2}\lambda_{\max}(A).
$$
This completes the proof.
\end{proof}

\newpage

\section{Illustrative behavior of the asymmetric $\alpha$-skew JSD family on univariate Gaussians}
\label{app:alphaJSD}

Figure~\ref{fig:lambda_full} provides an intuition for the interpolation behavior by illustrating how the mixture distribution $M_\alpha$ and the corresponding divergence $\mathrm{JSD}^\alpha_a(P \,\|\, Q)$ evolve with $\alpha$ for two Gaussian distributions $P = \mathcal{N}(-1,1)$ and $Q = \mathcal{N}(1,1)$. Figure~\ref{fig:lambda_distributions} shows how the mixture $M_\alpha = (1-\alpha)P + \alpha Q$ transitions smoothly from $P$ to $Q$ as $\alpha$ increases, while Figure~\ref{fig:lambda_divergence_curve} shows that $\mathrm{JSD}^\alpha_a(P \,\|\, Q)$ approaches zero at the endpoints and reaches its maximum at $\alpha=\tfrac{1}{2}$, corresponding to the symmetric JSD.

\begin{figure}[h!]
    \centering
    \begin{subfigure}[t]{0.48\linewidth}
        \centering
        \includegraphics[width=0.8\linewidth]{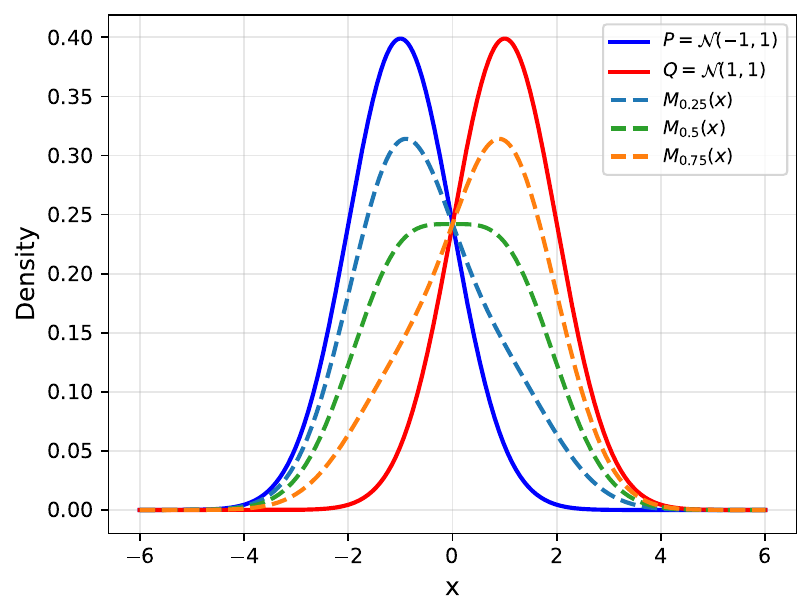}
        \caption{Distributions and their mixtures $M_\alpha = (1-\alpha)P + \alpha Q$.}
        \label{fig:lambda_distributions}
    \end{subfigure}
    \hfill
    \begin{subfigure}[t]{0.48\linewidth}
        \centering
        \includegraphics[width=0.8\linewidth]{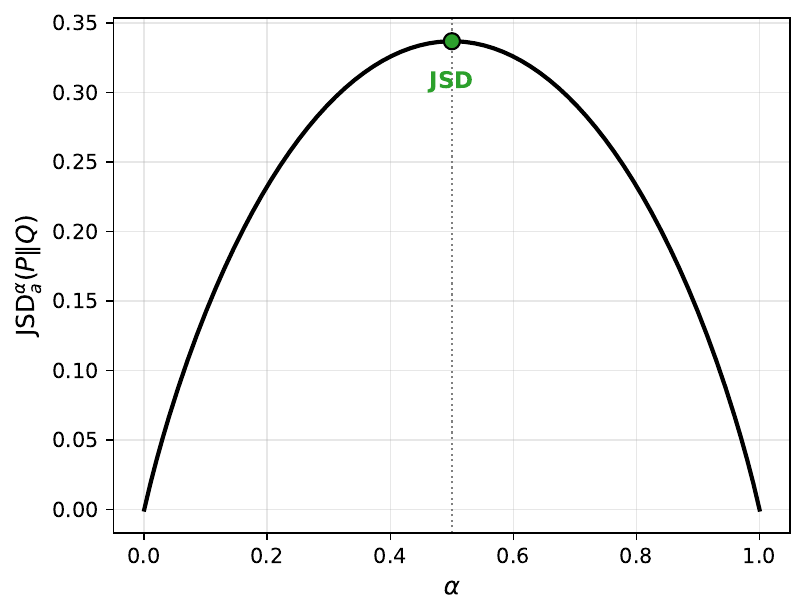}
        \caption{Corresponding divergence values $\mathrm{JSD}^\alpha_a(P \,\|\, Q)$ as a function of $\alpha$.}
        \label{fig:lambda_divergence_curve}
    \end{subfigure}
    \caption{Illustration of the asymmetric $\alpha$-skew Jensen--Shannon divergence family on a pair of univariate Gaussian distributions $P = \mathcal{N}(-1,1)$ and $Q = \mathcal{N}(1,1)$ across different $\alpha$ values.}
    \label{fig:lambda_full}
\end{figure}

\section{Detailed description of the transformation used}
\label{app:transformation}

\begin{figure}[h!]
    \centering
    \includegraphics[width=0.65\linewidth]{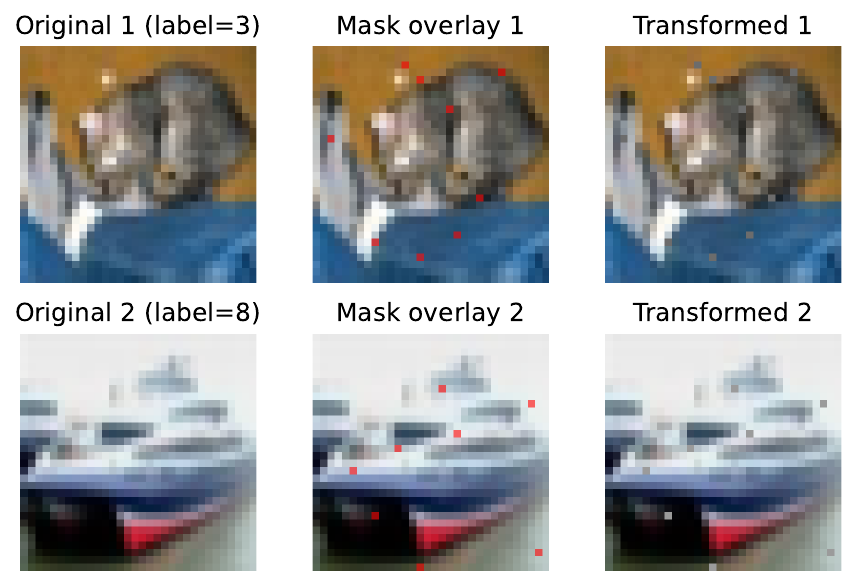}
    \caption{Two examples of input transformation.}
    \label{fig:transformation}
\end{figure}

For each input image, we generate a transformed version by randomly masking a small subset of pixels. Specifically, an independent spatial map is sampled uniformly in $[-1,1]$ at each pixel location. This map is defined directly at pixel resolution, without any blockwise aggregation. We then take its absolute value and select the pixels whose values fall below the $2\%$ quantile. These selected pixels are replaced by a baseline equal to the mean intensity of the corresponding image. 

\newpage
\section{Additional empirical analyses}
\subsection{Symmetric power-law characterization of the empirical curves}
\label{app:powerlaw_fits}
To characterize the empirical shape of the performance curves, we analyze the mean metric values across different $\alpha$ by fitting a symmetric power-law function of the form
\begin{align*}
    f(\alpha) = a\,|\alpha - 0.5|^p + b,
\end{align*}
where $f(\alpha)$ denotes the empirical metric as a function of $\alpha$, and the exponent $p$ characterizes the sharpness of the variation around the symmetric point $\alpha = 0.5$. The parameters $a$ and $b$ act as scaling and offset terms, respectively, while the exponent $p$ captures the effective curvature of the empirical response. Smaller values of $p$ correspond to sharper variations near the center, whereas values close to $p \approx 2$ indicate an approximately quadratic dependence.

Figure~\ref{fig:shapes} illustrates the fitted symmetric power-law curves for each dataset and evaluation metric. For CIFAR-10, the curves exhibit pronounced concave and convex shapes ($p<2$), indicating sharper performance gains near the symmetric regime. EMNIST displays a milder variation around the center ($3<p<4$), suggesting a smoother trade-off across $\alpha$. Fashion-MNIST and Oxford-IIIT Pet exhibit nearly quadratic and symmetric trends ($p\approx2$), reinforcing that the symmetric divergence yields the most stable and well-calibrated behavior across diverse datasets.

\begin{figure*}[h!]
    \centering
    \includegraphics[width=\linewidth]{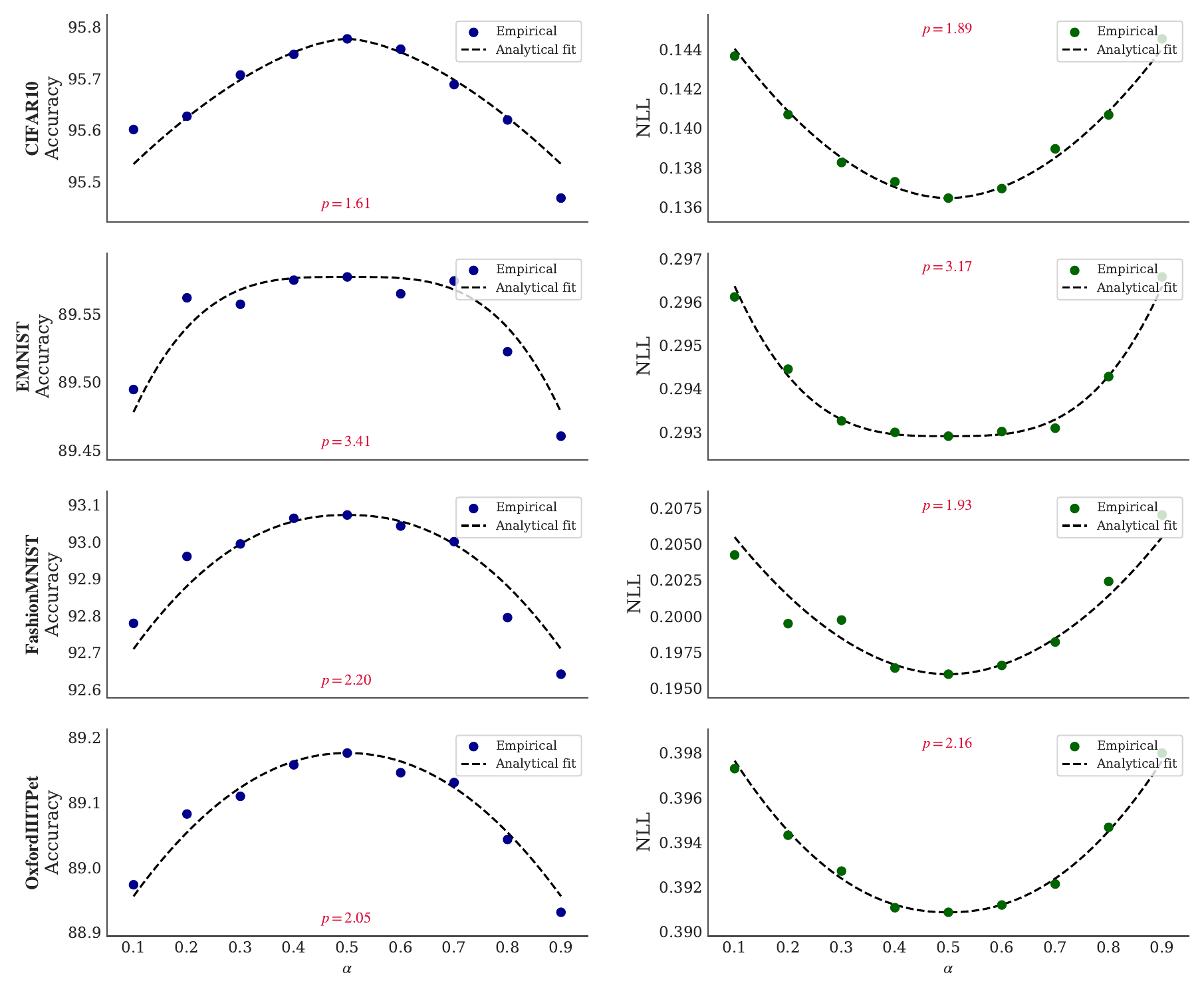}
    \caption{Analytical performance-curvature analysis using symmetric power-law fits.}
    \label{fig:shapes}
\end{figure*}

\newpage
\subsection{Two-dimensional loss landscape visualizations}
\label{app:landscape-2d}
 \begin{figure}[h!]
     \centering
     \includegraphics[width=\linewidth]{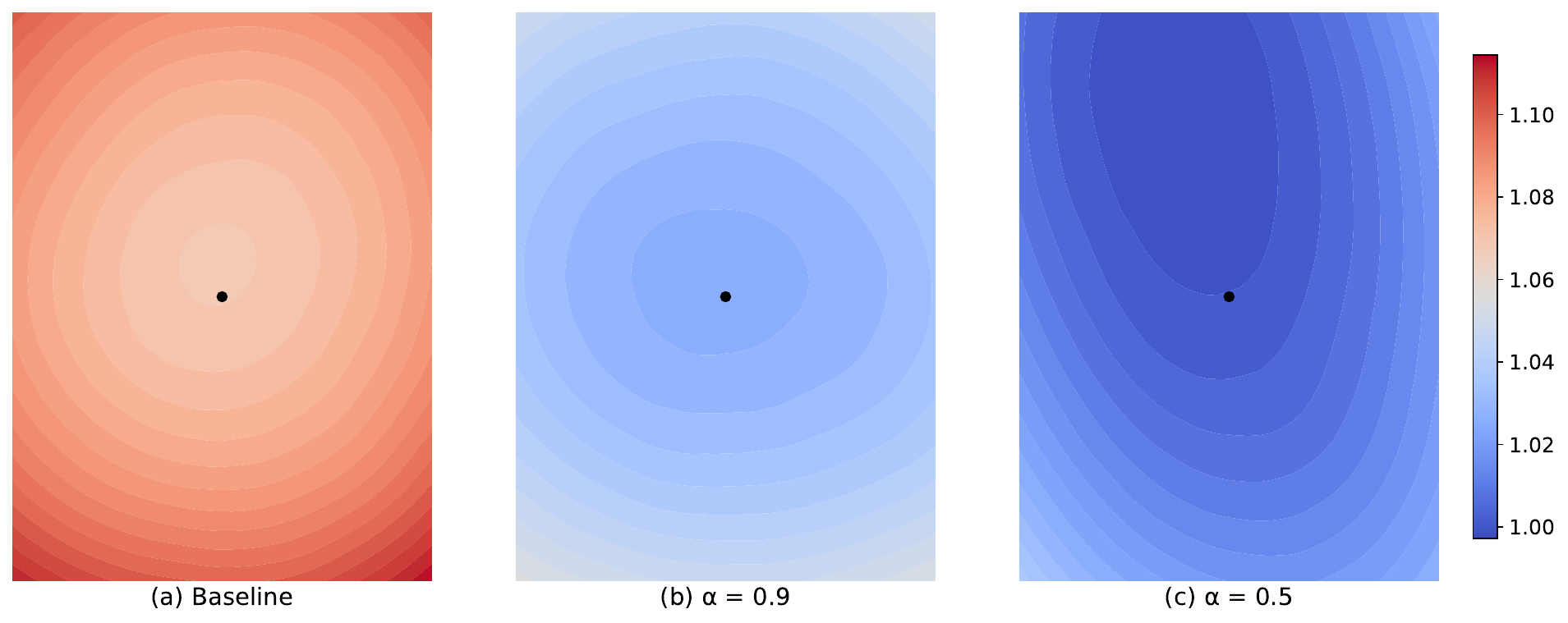}
     \caption{Two-dimensional visualizations of the training loss landscape around the final parameters for the baseline model and for models trained with asymmetric $\alpha$-skew JSD regularization at $\alpha=0.9$ and $\alpha=0.5$. All landscapes are displayed on a shared color scale.}
     \label{fig:loss_landscape_2d}
 \end{figure}

Figure~\ref{fig:loss_landscape_2d} presents two-dimensional slices of the training loss landscape around the final parameters for three models: an unregularized baseline, and models trained with asymmetric $\alpha$-skew JSD regularization for $\alpha=0.9$ and $\alpha=0.5$. All slices are evaluated on a common perturbation plane and displayed on the same color scale. The baseline exhibits the sharpest basin, while divergence-based regularization leads to progressively flatter local minima. In particular, the model trained with $\alpha=0.5$, corresponding to the strongest local curvature penalization within this family, displays the broadest and flattest basin. This qualitative behavior is consistent with the theoretical analysis, according to which stronger curvature penalization promotes flatter and less sharp local minima.

\end{document}

%% file: math_commands.tex
\usepackage{amsmath,amsfonts,bm}
\usepackage{amssymb}
\usepackage{amsthm}

\newtheorem{proposition}{Proposition}

\newtheorem{corollary}{Corollary}
\newtheorem{remark}{Remark}

\def\eqref#1{equation~\ref{#1}}

\def\1{\bm{1}}

\DeclareMathAlphabet{\mathsfit}{\encodingdefault}{\sfdefault}{m}{sl}
\SetMathAlphabet{\mathsfit}{bold}{\encodingdefault}{\sfdefault}{bx}{n}

